\documentclass{sig-alternate}
\usepackage{amsmath}
\usepackage{graphicx}
\usepackage{tabularx}
\usepackage{algorithm}
\usepackage{algorithmic}
\usepackage{times}
\usepackage{mdwlist}
\usepackage[sort&compress,numbers]{natbib}
\usepackage[numbers]{natbib}
\usepackage{hyperref}

\ifx\theorem\undefined
\newtheorem{theorem}{Theorem}
\fi

\ifx\lemma\undefined
\newtheorem{lemma}[theorem]{Lemma}
\fi

\begin{document}

\conferenceinfo{KDD'14,} {August 24--27, 2014, New York City, New York USA.}
\CopyrightYear{2014} 
\crdata{978-1-4503-1462-6 /12/08} 
\clubpenalty=10000 
\widowpenalty = 10000

\title{The Structurally Smoothed Graphlet Kernel}

\numberofauthors{2} 

\author{ 
  \alignauthor
  Pinar Yanardag\\
  \affaddr{Department of Computer Science}\\
  \affaddr{Purdue, West Lafayette IN}\\
  \email{ypinar@purdue.edu}
  \alignauthor
  S.V.\,N. Vishwanathan\\
  \affaddr{Dept.\ of Statistics and Computer Science}\\
  \affaddr{Purdue, West Lafayette IN}\\
  \email{vishy@stat.purdue.edu}
}
\date{\today}

\maketitle

\begin{abstract} 
  A commonly used paradigm for representing graphs is to use a vector
  that contains normalized frequencies of occurrence of certain motifs
  or sub-graphs. This vector representation can be used in a variety of
  applications, such as, for computing similarity between graphs. The
  graphlet kernel of \citet{SheVisPetMehetal09} uses induced sub-graphs
  of $k$ nodes (christened as graphlets by \citet{Przulj06}) as motifs
  in the vector representation, and computes the kernel via a dot
  product between these vectors. One can easily show that this is a
  valid kernel between graphs. However, such a vector representation
  suffers from a few drawbacks. As $k$ becomes larger we encounter the
  sparsity problem; most higher order graphlets will not occur in a
  given graph. This leads to diagonal dominance, that is, a given graph
  is similar to itself but not to any other graph in the dataset.  On
  the other hand, since lower order graphlets tend to be more numerous,
  using lower values of $k$ does not provide enough discrimination
  ability. We propose a smoothing technique to tackle the above
  problems. Our method is based on a novel extension of Kneser-Ney and
  Pitman-Yor smoothing techniques from natural language processing to
  graphs. We use the relationships between lower order and higher order
  graphlets in order to derive our method. Consequently, our smoothing
  algorithm not only respects the dependency between sub-graphs but also
  tackles the diagonal dominance problem by distributing the probability
  mass across graphlets. In our experiments, the smoothed graphlet
  kernel outperforms graph kernels based on raw frequency counts.

\end{abstract} 

\section{Introduction}
\label{submission}

In this paper, we are interested in comparing graphs by computing a
kernel between graphs \citep{VisSchKonBor10}. Graph kernels are popular
because many datasets from diverse domains such as bio-informatics
\citep{ShaIde06,BorOngSchVisetal05}, chemo-informatics \citep{BonRou91},
and web data mining \citep{WasMot03} naturally can be represented as
graphs. Almost all graph kernels (implicitly or explicitly) represent a
graph as a (normalized or un-normalized) vector which contains the
frequency of occurrence of motifs or sub-graphs\footnote{The kernels
  proposed by \cite{KonBor08} are a notable exception.}. The key idea
here is that well chosen motifs can capture the semantics of the graph
structure while being computationally tractable. For instance, counting
walks in a graph leads to the random walk graph kernel of
\citet{BorOngSchVisetal05} (see \citet{VisSchKonBor10} for an efficient
algorithm for computing this kernel). Other popular motifs include
subtrees \citep{SheBor10}, shortest paths \citep{BorKri05}, and cycles
\citep{HorGaeWro04}. Of particular interest to us are the graphlet
kernels of \citet{SheVisPetMehetal09}. The motif used in this kernel is
the set of unique sub-graphs of size $k$, which were christened as
graphlets by \citet{Przulj06}.

\begin{description}
\item[Observation 1: ] Computing meaningful graphlet kernels that have
  high discriminative ability requires a careful selection of $k$. If
  $k$ is small, then the number of unique graphlets is small (i.e., the
  length of the feature vector is small). See Figure
  \ref{fig:unique-graphlets}. Consequently the feature vector does not
  provide meaningful discrimination between two graphs. On the other
  hand, if $k$ is large then a) the set of unique graphlets grows
  exponentially (i.e., the feature vector is very high dimensional) but
  b) only a small number of unique graphlets will be observed in a given
  graph (i.e., the feature vector is very sparse). Moreover, the
  probability that two graphs will contain a given large sub-graph is
  very small. Consequently, a graph is similar to itself but not to any
  other graph in the training data. This is well known as the diagonal
  dominance problem in the machine learning community
  \citep{KanGraSha03}, and the resulting kernel matrix is close to the
  identity matrix. In other words, the graphs are orthogonal to each
  other in the feature space. However, it is desirable to use large
  values of $k$ in order to gain better discriminative ability. One way
  to circumvent the diagonal dominance problem is to view the normalized
  graphlet-frequency vector as estimating a multinomial distribution,
  and use smoothing.
\item[Observation 2:] The normalized graphlet-frequency vector exhibits
  power-law behavior, especially for large values of $k$. In other
  words, a few popular graphlets occur very frequently while a vast
  majority of graphlets will occur very rarely. Put another way, a few
  graphlets dominate the distribution. To see this, we randomly sampled
  a graph from six benchmark datasets (details of the datasets can be
  found in Section~\ref{sec:Experiments}), and exhaustively computed
  occurrences of all graphlets of size $k=8$ and plotted the resulting
  histogram on a log-log scale in Figure~\ref{fig:powerlaw}.  As can be
  seen, the frequencies are approximately linear in the log-log scale
  which indicates power law behavior. Therefore, any smoothing technique
  that we use on the normalized graphlet-frequency vector must respect
  this power-law behaviour.
\item[Observation 3:] The space of graphlets is structured. What we mean
  by this is that graphlets of different sizes are related to each
  other. While many such relationships can be derived, we will work with
  perhaps the simplest one which is depicted in
  Figure~\ref{fig:dag}. Here, we construct a directed acyclic graph
  (DAG) with the following property: a node at depth $k$ denotes a
  graphlet of size $k$. Given a graphlet $g$ of size $k$ and other
  graphlet $g'$ of size $k+1$ we add an edge from $g$ to $g'$ if, and
  only if, $g$ can be obtained from $g'$ by deleting a node of $g'$.
  This shows that graphlets of size $k$ have a strong relationship to
  graphlets of size $k+1$ and one must respect this relationship when
  deriving a smoothing technique.
\end{description}

\begin{figure}
  \centering
  \centerline{\includegraphics[width=1.0\columnwidth]{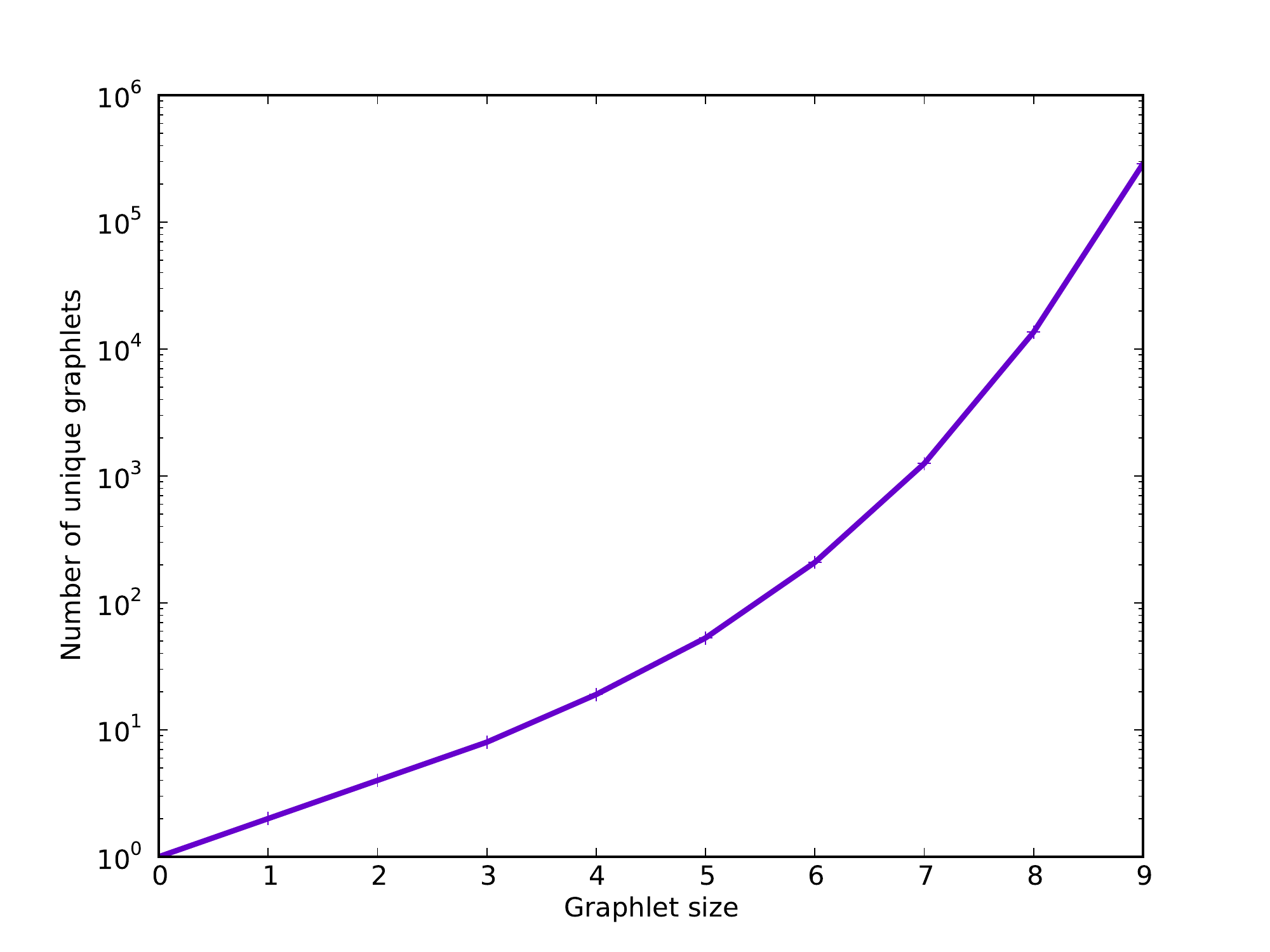}}
  \caption{Number of unique graphlets increase exponentially with
    graphlet size $k$. }
  \label{fig:unique-graphlets}
\end{figure}

\begin{figure}
  \centering
  \centerline{\includegraphics[width=1.0\columnwidth]{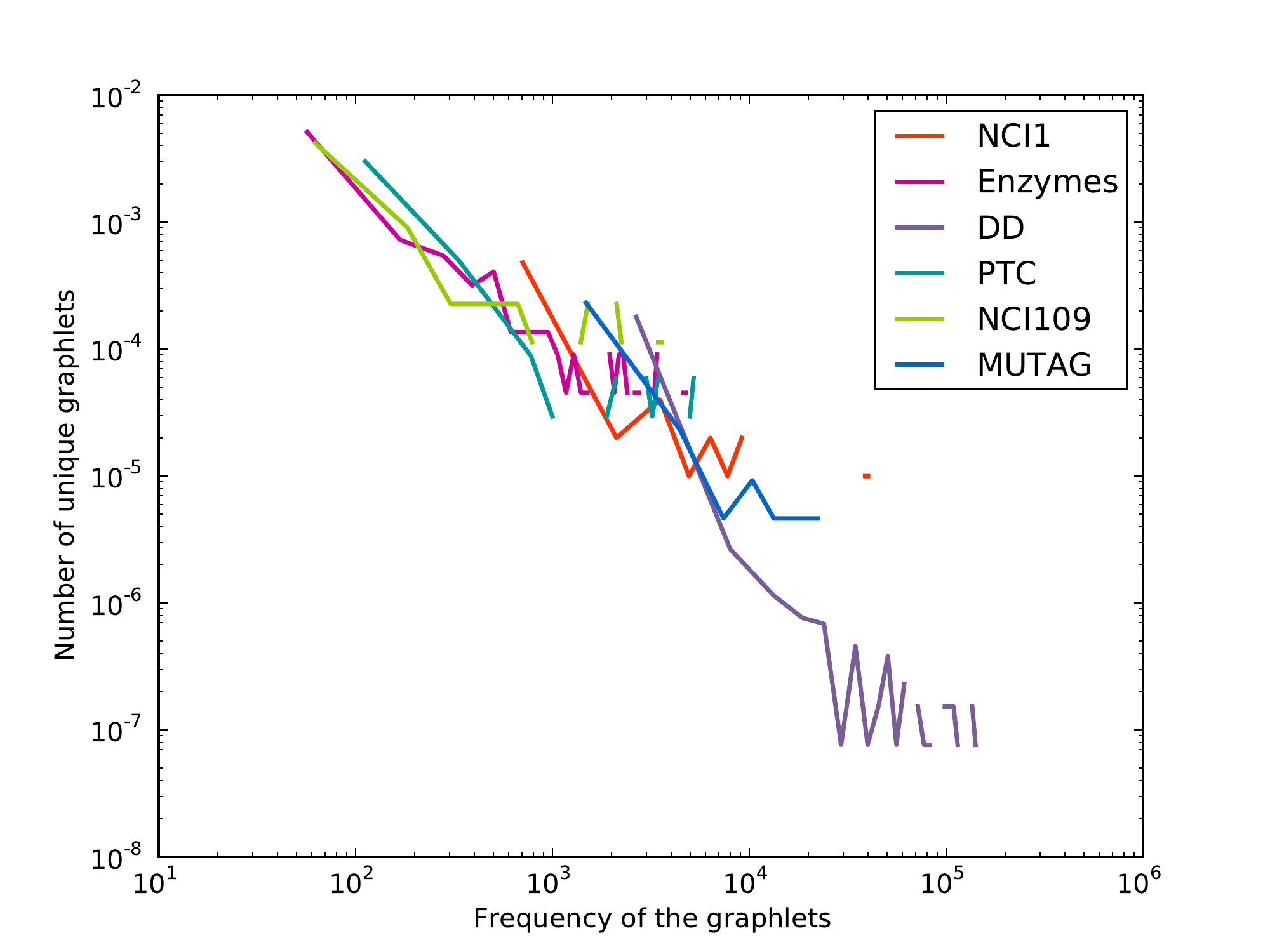}}
  \caption{We randomly select a graph from six benchmark graph datasets
    and exhaustively searched for all graphlets of size $k=8$ . The
    histogram is plotted in log-log scale in order to demonstrate the
    power-law behaviour. }
  \label{fig:powerlaw}
\end{figure}

Our contributions in this paper are as follows. First, we propose a new
smoothing technique for graphlets which is inspired by Kneser-Ney
smoothing~\citep{KneNey95} used for language models in natural language
processing. Our model satisfies the desiderata that we outlined above,
that is, it respects the power law behavior of the counts and yet takes
into account the structure of the space of graphlets. Second, we provide
a novel Bayesian version of our model that is extended from the
Hierarchical Pitman-Yor process of \citet{Teh06b}. Unlike the
traditional Hierarchical Pitman-Yor Process (HPYP) where the base
distribution is given by another Pitman-Yor Process (PYP), in our case
it is given by a transformation of a PYP that is guided by the structure
of the space. Third, we perform experiments to validate and understand
how smoothing affects the performance of graphlet kernels.

The structure of the paper is as follows. In Section
\ref{sec:Background}, we discuss background on graphlet kernels and
smoothing techniques. In Section \ref{sec:Definbasedistr}, we introduce
our Kneser-Ney-inspired smoothing technique. In Section
\ref{sec:PitmanYor}, we propose an alternate Bayesian version of our
model. Related work is discussed in Section \ref{sec:RelatedWork}. In
Section \ref{sec:Experiments}, we perform experiments and discuss
our findings, and we conclude the paper with Section
\ref{sec:Discussion}.

\section{Background}
\label{sec:Background}

\subsection{Notation}
\label{sec:notation}
A {\em graph} is a pair $G=(V,E)$ where $V = \left \{ v_1, v_2, \ldots,
  v_{|V|} \right \}$ is an ordered set of \emph{vertices} or
\emph{nodes} and $E \subseteq V \times V$ is a set of \emph{edges}.
Given $G = (V, E)$ and $H = (V_H , E_H )$, $H$ is a {\em sub-graph} of
$G$ iff there is an injective mapping $\alpha : V_H \rightarrow V$ such
that $(v, w) \in E_H$ iff $(\alpha(v), \alpha(w)) \in E$.  Two graphs $G
= (V, E)$ and $G' = (V', E')$ are {\em isomorphic} if there exists a
bijective mapping $g: V \rightarrow V'$ such that $(v_i, v_j) \in E$ iff
$(g(v_i), g(v_j)) \in E'$. {\em Graphlets} are small, connected,
non-isomorphic sub-graphs of a large network. They were introduced by
\citet{Przulj06} to design a new measure of local structural similarity
between biological networks. Graphlets up to size five are shown in
Figure \ref{fig:graphlets}.

\begin{figure*}
  \centering
  \centerline{\includegraphics[width=2.1\columnwidth]{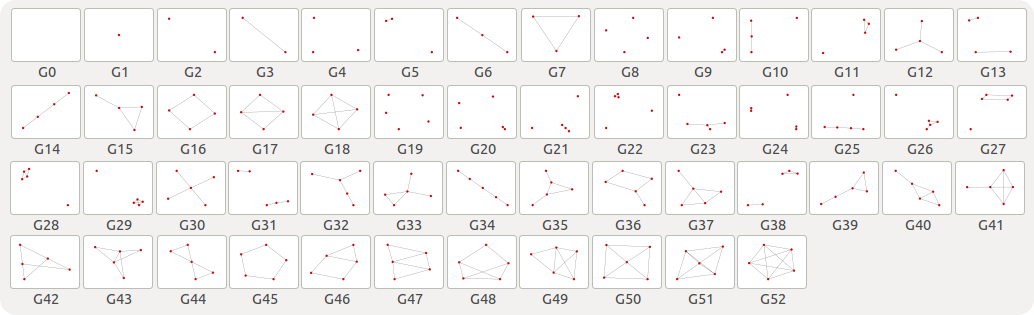}}
  \caption{ Connected, non-isomorphic induced sub-graphs of size $k\leq
    5$. Plots are generated with NetworkX library \cite{HagSwaChu08}.}
  \label{fig:graphlets}
\end{figure*}

\subsection{The Graphlet kernel}
\label{sec:gk}
Let $\mathcal{G}_{k} = \{ g_{1}, g_2, \ldots, g_{n_k} \}$ be the set of
size-$k$ graphlets where $n_k$ denotes the number of unique graphlets of
size $k$.  Given a graph $G$, we define $f_{G}$ as a normalized vector
of length $n_k$ whose $i$-th component corresponds to the frequency of
occurrence of $g_{i}$ in $G$:
\begin{align}
  \label{eq:gk}
  f_G = (\frac{c_1}{\sum_{j}^{n_k} c_j}, \cdots, \frac{c_{n_k}}{\sum_{j}^{n_k} c_j})^T.
\end{align}
Here $c_i$ denotes number of times $g_i$ occurs as a sub-graph of
$G$. Given two graphs $G$ and $G'$, the graphlet kernel $k_{g}$ is
defined as:
\begin{align}
  \label{eq:graphlet-kernel}
  k_{g}(G, G'):= f_{G}^{\top} f_{G'},
\end{align} 
which is simply the dot product between the normalized
graphlet-frequency vectors.

\subsection{Smoothing multinomial distributions}
\label{sec:smoothing}

In this section we will briefly review smoothing techniques for
multinomial distributions and show that graphlet kernels are indeed
based on estimating a multinomial. Suppose we observe a sequence $e_1,
e_2, \ldots, e_n$ containing $n$ discrete events drawn from a ground set
of size $M$ and we would like to estimate the probability $P(e_{i})$
of observing each event $e_{i}$. Maximum likelihood estimation based on
sequence counts obtained from the observations provides a way to compute
$P(e_i)$:
\begin{align}
  \label{eq:mle}
  P_{MLE}(e_i) = \frac{c_{i}}{\sum_{j} c_{j}},
\end{align}
where $c_{i}$ denotes the number of times the event $e_{i}$ appears in
the observed sequence and $\sum_{j} c_{j}$ denotes the total number
of observed events. Therefore, one can easily see that the representation
used in graphlet kernels in Section \ref{sec:gk} is
actually an MLE estimate on the observed sequences of graphlets.

However, MLE estimates of the multinomial distribution are \emph{spiky},
that is, they assign zero probability to events that did not occur in
the observed sequence. What this means is that an event with low
probability is often estimated to have zero probability mass. This issue
occurs in a number of different domains and therefore, unsurprisingly,
has received significant research attention \citep{ZhaLaf04}; smoothing
methods are typically used to address this problem. The general idea
behind smoothing is to discount the probabilities of the observed events
and to assign extra probability mass to unobserved events.

Laplace smoothing is the simplest and one of the oldest smoothing
methods, where only a fixed count of 1 is added to every event. This
results in the estimate
\begin{align}
  \label{eq:laplace}
  P_{Laplace}(e_{i}) = \frac{\sum_{j} c_{j}}{\sum_{j} c_{j} + M} P_{MLE} +
  \frac{M}{\sum_j c_j + M}\frac{1}{M}  
\end{align}
or equivalently, 
\begin{align}
  \label{eq:laplace-rewrite}
  P_{Laplace}(e_{i}) = \lambda P_{MLE}(e_{i}) + (1-\lambda)\frac{1}{M},
\end{align}
where $\lambda$ is a normalization factor which ensures that the
distributions sum to one.  The intuition behind Laplace smoothing is
basically to interpolate a uniform distribution with the MLE
distribution. Although Laplace smoothing resolves the zero-count problem,
it does not produce a power law distribution which is a desirable
feature in real-life models. Therefore, researchers have worked on
finding smoothing techniques that respect power law behavior. The key
idea behind these methods is to redistribute the probability mass using
a so-called \emph{fallback} model, where the fallback model is also
recursively estimated. 

Kneser-Ney smoothing is a fallback based smoothing method which has been
identified as the state-of-the-art smoothing in natural language
processing by several studies \citep{GooChe96}. Kneser-Ney smoothing 
computes the probability of an event by using the raw counts that are
discounted by using a fixed mass. Then, the discounted mass is
re-added equally to all event probabilities by using a base
distribution: 
\begin{align}
  \label{eq:kn}
  P_{KN}(e_i) = \frac{max \{ c_i -d, 0\}}{\sum_{j} c_j} + \sum_{j=1}^n
  |\{e_j: c_j > d \}|\frac{d}{\sum_{j} c_j} P_{0}(e_i),
\end{align}
where $d \geq 0 $ is the discounting parameter, $P_{0}(\cdot)$ is the
base distribution and $P_0(e_i)$ denotes the probability mass the base 
distribution assigned to event $e_i$.  The quantity $\sum_{j=1}^n
|\{e_j: c_j > d \}|$ is a normalization factor to ensure that the
distribution sums to 1 and simply denotes the number of events the
discount is applied. When discount parameter $d=0$, we recover MLE
estimation since no mass is taken away from any event. When $d$ is very
large, then we recover the base distribution on the events since we
discount all the available mass. One should interpolate between these
two extremes in order to get a reasonable smoothed estimate. In order to
propose a new Kneser-Ney-based smoothing framework, one needs to specify
the discount parameter $d$ and the base distribution $P_{0}(\cdot)$. In
the next section we will show how one can derive a meaningful base
distribution for graphlets.

\section{Defining a base distribution}
\label{sec:Definbasedistr}

\begin{figure*}
  \centering
  \centerline{\includegraphics[width=2.1\columnwidth]{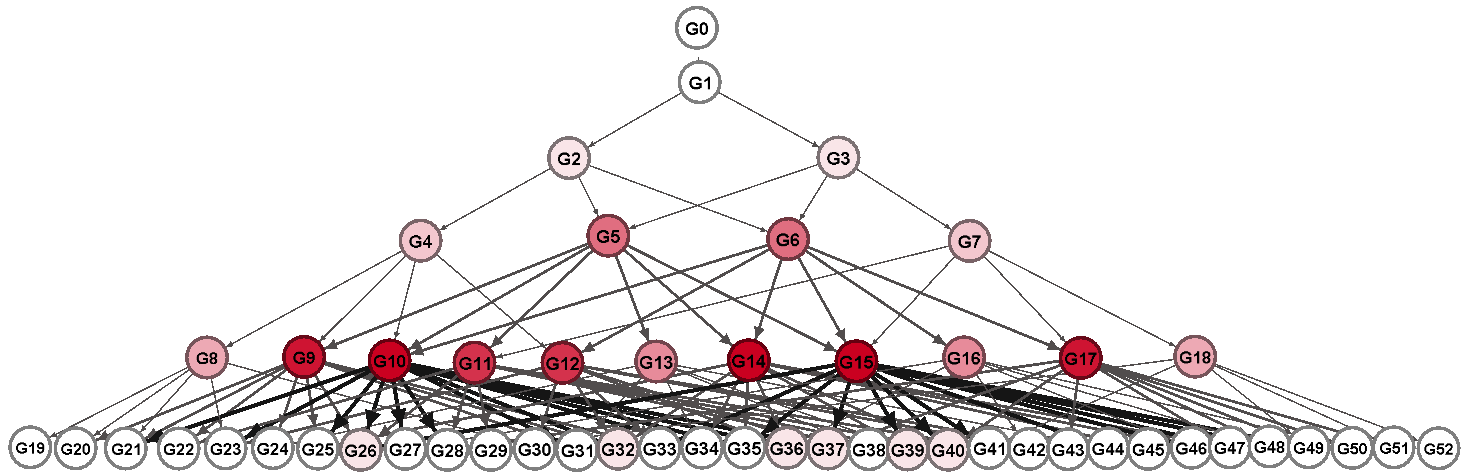}}
  \caption{Topologically sorted graphlet DAG where each node are within
    one delete-distance away. Nodes are colored by average degree. Image
    is generated by Gephi \cite{BasHeyJac09}.} 
  \label{fig:dag}
  \vskip -0.2in
\end{figure*}

The space of graphlets has an inherent structure. One can construct a
directed acyclic graph (DAG) in order to show how different graphlets
are related with each other. A node at depth $k$ denotes a graphlet of
size $k$. When it is clear from context, we will use the node of the DAG
and a graphlet interchangeably. Given a graphlet $g_i$ of size $k$ and
another graphlet $g_j$ of size $k+1$ we add an edge from $g_i$ to $g_j$
if, and only if, $g_i$ can be obtained from $g_j$ by deleting a node of
$g_j$. We first discuss how to construct this DAG and then discuss how
we can use this DAG to define a base distribution.

The first step towards constructing our DAG is to obtain all unique
graphlet types of size $k$. Therefore, we first exhaustively generate
all possible graphs of size $k$ (this involves a one time $O(2^k)$ effort), and
use Nauty \cite{McKay07} to obtain their canonically-labelled isomorphic
representations. In order to obtain the edges of the DAG, we take a node
at depth $k+1$, which denotes a canonically-labeled isomorphic graphlet
and delete a node to obtain a size $k$ graphlet. We use Nauty to
canonically-label the size $k$ graph, which in turn allows us to link to
a node at depth $k$. By deleting each node of the $k+1$ sized graphlet
we can therefore obtain $k+1$ possible links. We repeat this for all
nodes at level $k+1$ before proceeding to level $k$.  Figure
\ref{fig:dag} shows the constructed DAG for size $k=5$ graphlets. Since
all descendants of a given graphlet at level $k$ are at level $k+1$,
a topological ordering of the vertices is possible, and hence it is
easy to see that the resulting graph is a DAG. 

\begin{figure}[ht]
  \vskip 0.2in
  \begin{center}
    \centerline{\includegraphics[width=0.8\columnwidth]{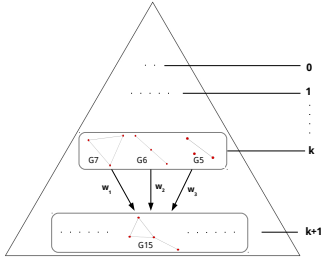}}
    \caption{Graphlet $g_{15}$ get the probability mass to its
      parents $g_{7}, g_{6}, g_{5}$ according to the
      weights $w_1, w_2, w_3$ respectively.} 
    \label{fig:triangle}
  \end{center}
\end{figure} 

Now, let us define the edge weight  between an arbitrary graphlet $g_j$
of size $k+1$ and its parent $g_i$ of size $k$. Let $s_{ij}$ denote the
number of times $g_i$ occurs as a sub-graph of $g_j$ and $\mathcal{C}_{g_{i}}$
denote all the children of graphlet $g_{i}$ in the DAG.  Then, we define
the edge weight between graphlet $g_i$ and $g_j$ as
\begin{align}
  \label{eq:weight}
  w_{ij} = \frac{s_{ij}}{\sum_{g_{j'} \in \mathcal{C}_{g_i}} s_{ij'}}
\end{align}

Next we show how the DAG can be used to define a base
distribution. Suppose we have a distribution over graphlets of size
$k$. Then we can transform it into a distribution over size $k+1$
graphlets in a recursive way by exploiting edge connections in the
DAG as follows:

\begin{align}
  \label{eq:base}
  P_{0}(g_j) = \sum_{g_i \in Pa(g_j)} w_{ij} P_0(g_i).
\end{align}
Here $g_j$ denotes a graphlet of size $k+1$ and
$Pa(g_{j})$ denotes the parents of graphlet $g_{j}$ in the DAG.
\begin{lemma}
  Given the set of child nodes of a graphlet $g_{i}$, if
  the edge weights on the DAG are all non-negative and satisfy
  \begin{align}
    \sum_{g_j \in C(g_i)} w_{ij} = 1,
  \end{align}
  then \eqref{eq:base} defines a valid probability distribution.
\end{lemma}
\begin{proof}
  For clarity, we introduce a few notations to facilitate the proof. Assume that there are in total $J$ child graphlets, which we denote by $g_1, g_2, \cdots, g_J$. Further assume that there are in total $I$ parent graphlets, which we denote by $g_1, \cdots, g_I$. Let $I_j$ denote the number of parents graphlet $g_j$ has, i.e. $I_j = |Pa(g_j)|$. Clearly, we have that $\sum_{j=1}^J I_j = I$. Thus, we have
  \begin{align*}
  \sum_{j=1}^J P_0(g_j) & = \sum_{j=1}^J [\sum_{i=1}^{I_j} w_{ij} P_0(g_i)] \\
  & = \sum_{i=1}^I [\sum_{\{j:g_j\in C(g_i)\}} w_{ij} P_0(g_i)] \\
  & = \sum_{i=1}^I P_0(g_i) [\sum_{\{j:g_j\in C(g_i)\}} w_{ij}]\\
  & = \sum_{i=1}^I P_0(g_i) = 1,
  \end{align*}
which completes the proof.
\end{proof}

The base distribution we defined above respects the structural space of
the graphlets. Pretend for a moment that we are given only the
frequencies of occurrences of size $k$ graphlets and are asked to infer
the probability of occurrences of size $k+1$ graphlets. Without any
additional information, one can infer the distribution as follows: each
parent graphlet casts a vote for every child graphlet based on how many
times the parent occurs in the child. The votes of all parents are
accumulated and this provides a distribution over size $k+1$
graphlets. In other words, a natural way to infer the distribution at
level $k+1$ is to use how likely we are to see its sub-graphs. Figure
\ref{fig:triangle} illustrates the relationship between a graphlet of
size $k+1$ and its parent graphlets of size $k$. Here, edge weights
denote how many times each parent occurs as a sub-graph of $g_{15}$. In
the case that we do not observe graphlet $g_{15}$, it still gets 
probability mass proportional to the edge weight from its parents $g_{7}, g_{6}, g_{5}$,
 thus overcoming the sparsity problem of unseen data. Our model
combines this base distribution with the observed real data to generates
the final distribution.

The way how the discounted mass is distributed is controlled by the edge
weights between two graphlets. In Equation \ref{eq:weight} we defined
edge weights according to the number of times a parent node occurs in
its children. However, one can explore different weighting schemes
between the nodes on the DAG based on domain knowledge. For example, in
the case of structured graphs such as social networks, one might
benefit from weighting the edges according to the PageRank
\cite{PagBriMotWin98} score of the nodes. Similarly, other link analysis
 algorithms such as Hubs or Authority given by HITS algorithm \cite{Kleinberg99} 
 can be used in order to exploit the domain knowledge.
 
\subsection{Kneser-Ney Smoothing with a structural distribution}
\label{sec:kn}

We now have all the components needed to explain our Structural
Kneser-Ney (SKN) framework. Given an arbitrary graphlet $g_j$ of size
$k+1$, we estimate the probability of observing that graphlet as
follows:

\begin{multline}
  P_{SKN}(g_j) = \frac{max(c_j -d, 0)}{\sum_{g_{j'} \in
      \mathcal{G}_{k+1}} c_{j'}} + \frac{d}{\sum_{g_{j'} \in
      \mathcal{G}_{k+1}} c_{j'}}
  \\
  \sum_{g_{j'} \in \mathcal{G}_{k+1}}|\{g_{j'}: c_{j'} > d \}| \sum_{i
    \in \mathcal{P}_{g_j}} P_{0}(g_i)\frac{w_{ij}}{\sum_{g_{j'} \in
      \mathcal{C}_{g_{i}}}w_{ij'}} 
  \label{eq:kn-graphlet}
\end{multline}
As can be seen from the equation, we first discount the count of all
graphlets by $d$, and then redistribute this mass to all other
graphlets. The amount of mass a graphlet receives is controlled by the
base distribution. In order to automatically tune the discount parameter
$d$, we use the Pitman-Yor process (a Bayesian approximation of
Kneser-Ney) in the next section.

\section{Pitman-Yor Process}
\label{sec:PitmanYor}

\begin{figure*}[ht]
  \begin{center}
    \centerline{\includegraphics[width=2.0\columnwidth]{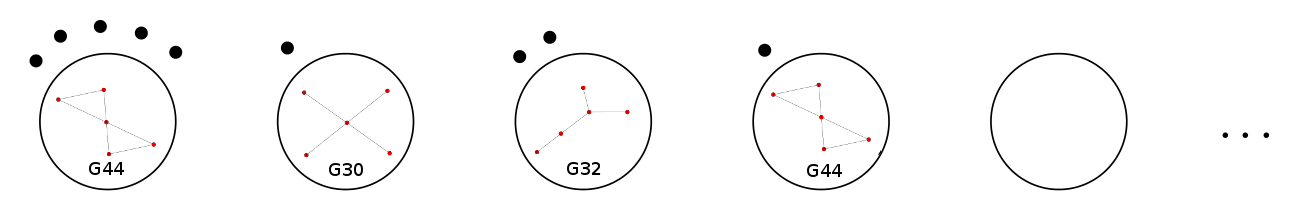}}
    \caption{An illustration of table assignment, adapted from
      \citet{GolGriJoh11}. In this example, labels at the tables are
      given by $(l_1 , \ldots, l_4 ) = (G_{44}, G_{30}, G_{32},
      G_{44})$. Black dots indicate the number of occurrences of each
      label in 10 draws from the Pitman-Yor process. }
    \label{fig:pyp}
  \end{center}
\end{figure*} 



We will only give a very high level overview of a Pitman-Yor process and
refer the reader to the excellent papers by \citet{Teh06b} and
\cite{GolGriJoh06} for more details.  A Pitman-Yor process $P$ on a ground
set $\mathcal{G}_{k+1}$ of size-$(k+1)$ graphlets is defined via
\begin{align}
  \label{eq:pyp-def}
  P_{k+1} \sim PY(d_{k+1}, \theta_{k+1}, P_{k}),
\end{align}
where $d_{k+1}$ is a discount parameter $0 \leq d_{k+1} < 1$, $\theta >
-d_{k+1}$ is a strength parameter, and $P_{k}$ is a base
distribution. The most intuitive way to understand draws from the
Pitman-Yor process is via the Chinese restaurant process (also see
Figure~\ref{fig:pyp}). Consider a restaurant with an infinite number of
tables. Customers enter the restaurant one by one. The first customer
sits at the first table, and is seated at the first table. Since this
table is occupied for the first time, a graphlet is assigned to it by
drawing a sample from the base distribution. The label of the first
table is the first graphlet drawn from the Pitman-Yor
process. Subsequent customers when they enter the restaurant decide to
sit at an already occupied table with probability proportional to $c_{i}
- d_{k+1}$, where $c_{i}$ represents the number of customers already
sitting at table $i$. If they sit at an already occupied table, then the
label of that table denotes the next graphlet drawn from the Pitman-Yor
process. On the other hand, with probability $\theta_{k+1} + d_{k+1} t$,
where $t$ is the current number of occupied tables, a new customer might
decide to occupy a new table. In this case, the base distribution is
invoked to label this table with a graphlet. Intuitively the reason this
process generates power-law behavior is because popular graphlets which
are served on tables with a large number of customers have a higher
probability of attracting new customers and hence being generated
again. This self reinforcing property produces power law behavior.

In a hierarchical Pitman-Yor process, the base distribution $P_{k}$ is
recursively defined via a Pitman-Yor process $P_{k} \sim PY(d_{k},
\theta_{k}, P_{k-1})$. In order to label a table, we need a draw from
$P_{k}$, which is obtained by inserting a customer into the
corresponding restaurant. 

In our case $P_{k+1}$ is defined over $\mathcal{G}_{k+1}$ of size
$n_{k+1}$ while $P_{k}$ is defined over $\mathcal{G}_{k}$ of size $n_{k}
\leq n_{k+1}$. Therefore, like we did in the case of Kneser-Ney
smoothing we will use the DAG and Equation \eqref{eq:base} to define a
base distribution. This changes the Chinese Restaurant process as
follows: When we need to label a table, we will first draw a size-$k$
graphlet $g_{i} \sim P_{k}$ by inserting a customer into the
corresponding restaurant. Given $g_{i}$, we will draw a size-$(k+1)$
graphlet $g_{j}$ proportional to $w_{ij}$, where $w_{ij}$ is obtained
from the DAG. Deletion of a customer is handled similarly. Detailed
pseudo-code can be found in Algorithms~\ref{alg:insert} and
\ref{alg:delete}. 

\renewcommand{\algorithmicrequire}{\textbf{Input:}}
\renewcommand{\algorithmiccomment}[1]{\hspace{0.in} // #1}
\begin{algorithm}
  \begin{algorithmic}
    \REQUIRE $d_{k+1}$, $\theta_{k+1}$, $P_{k}$
    \STATE $t \leftarrow 0$  \COMMENT{Occupied tables}
    \STATE $c \leftarrow ()$ \COMMENT{Counts of customers}
    \STATE $l \leftarrow ()$ \COMMENT{Labels of tables}
    \IF{$t=0$}     
    \STATE $t \leftarrow 1$
    \STATE append 1 to $c$
    \STATE draw graphlet $g_{i} \sim P_{k}$ \COMMENT{Insert customer in
      parent}
    \STATE draw $g_{j} \sim w_{ij}$ 
    \STATE append $g_{j}$ to $l$
    \RETURN $g_{j}$
    \ELSE
    \STATE with probability $\propto \max(0, c_j - d)$
    \STATE $c_j \leftarrow c_j+1$
    \RETURN $l_j$
    \STATE with probability proportional to $\theta + d t$
    \STATE $t \leftarrow t+1$
    \STATE append 1 to $c$
    \STATE draw graphlet $g_{i} \sim P_{k}$ \COMMENT{Insert customer in
      parent}
    \STATE draw $g_{j} \sim w_{ij}$ 
    \STATE append $g_{j}$ to $l$
    \RETURN $g_{j}$
    \ENDIF
  \end{algorithmic}
  \caption{Insert a Customer}
  \label{alg:insert}
\end{algorithm}
\begin{algorithm}
  \begin{algorithmic}
    \REQUIRE $d$, $\theta$, $P_{0}$, $C$, $L$, $t$
    \STATE with probability $\propto c_l$
    \STATE $c_l \leftarrow c_l-1$
    \STATE $g_j \leftarrow l_j$
    \IF{$c_l=0$}
    \STATE $P_k$ $\propto$ $1/w_{ij}$
    \STATE delete $c_l$ from $c$
    \STATE delete $l_j$ from $l$
    \STATE $t \leftarrow t - 1$
    \ENDIF
    \RETURN $g$
  \end{algorithmic}
  \caption{Delete a Customer}
  \label{alg:delete}
\end{algorithm}

\section{Related Work}
\label{sec:RelatedWork}

The problem of estimating multinomial distributions is a classic
problem. In natural language processing they occur in the following
context: suppose we are given a sequence of words $w_{1}, \ldots, w_{k}$
and one is interested in asking what is the probability of observing
word $w$ next. Estimating this probability lies at the heart of language
models, and many sophisticated smoothing techniques have been
proposed. This is a classic multinomial estimation problem that suffers
from sparsity since the event space is unbounded. Moreover, natural
language exhibits power law behavior since the distribution tends to be
dominated by a small number of frequently occurring words.  In extensive
empirical evaluation it has been found the Kneser-Ney smoothing is very
effective for language models \cite{GooChe96}, \cite{ManRagSch08}. Here,
the base distribution is constructed using smaller context of $k-1$
words which naturally leads to a denser distribution. Even though
language models and graphlets have some similarities, there is a
significant fundamental difference between the two. In language models,
one can derive the base distribution using a smaller context. However,
in the case of graphlets there is no equivalent concept of a fallback
model. Therefore, we need to derive the base distribution by using
\textit{smaller size} graphlets. However, this leads to a problem since
the distribution is now defined on a smaller space. Therefore, we need
to apply a transformation by using the DAG in order to convert the
distribution back into to the original space. \citet{GolGriJoh11} and
\citet{Teh06b} independently showed that Kneser-Ney can be explained in
a Bayesian setting by using the Pitman-Yor Process (PYP)
\cite{PitYor97}. In the Bayesian interpretation, a hierarchical PYP
where the Pitman-Yor prior comes from another PYP is used. Similar to
Kneser-Ney, this interpretation is not directly applicable to our model
since the previous PYP has a different space, thus we need to apply a
transformation.

Graph kernels can be considered as special cases of convolutional
kernels proposed by \citet{Haussler99}. In general, graph kernels can be
categorized into three classes: graph kernels based on walks and paths
\cite{GaeFlaWro03}, \cite{KasTsuIno04},\cite{BorKri05}, graph kernels
based on limited-size sub-graphs \cite{HorGaeWro04}, \cite{SheBor10},
\cite{ShiPetDroLanetal09} and graph kernels based on subtree patterns
\cite{RamGae03}. \citet{SheSchVanMehetal11} performs a relaxation on the
vertices and exploit labeling information embedded in the graphs to
derive their so-called Weisfeiler-Lehman kernels. However, their kernel
is applicable only to labeled graphs. The sparsity problem of graphlet
kernels has been addressed before.  Hash kernels proposed by
\citet{ShiPetDroLanetal09} addresses the sparsity problem by applying a
sparse projection into a lower dimensional space. The idea here is that
many higher order graphlets will ``collide'' and therefore be mapped to
the same lower dimensional representation, thus avoiding the diagonal
dominance problem. Unfortunately, we find that in our experiments the
hash kernel is very sensitive to the hash value used for embedding and
rarely performed well as compared to the MLE estimate. 


\section{Experiments}
\label{sec:Experiments}

\begin{table*}[t]
  \caption{Classification accuracy with standard deviation on benchmark
    data sets. RW: random walk kernel, SP: shortest path kernel, GK:
    graphlet kernel, HK: hash kernel, KN: Kneser-Ney smoothing, PYP:
    Pitman Yor smoothing, >24h: computation did not finish within 24
    hours. GK, KN and PYP results are reported for size 5 graphlets.} 
  \label{tab:results}
  \begin{center}
    \begin{tabular}{|c|c|c||c|c|c|c|c|}
      \hline
      Dataset & RW & SP & GK & HK & KN & PYP \\ 
      \hline  
      MUTAG &  83.51  & 78.72 & 80.34$\pm 3.0$  & 80.34$\pm 3.0$ & \textbf{82.98} $\pm 5.4$ & 81.94$\pm 2.9$\\\hline
      PTC & 51.16 & 50 & 57.26$\pm 4.6$  & 57.26$\pm 4.6$ & \textbf{59.87}$\pm 4.6$ & 56.36$\pm 4.3$\\\hline
      DD & > 24h & > 24h & 72.74$\pm 1.9$ &  72.74$\pm 1.9$& \textbf{74.95}$\pm 2.3 $& 73.51$\pm 2.8$\\\hline
      ENZYMES & 18.5 & 21.66 &19.50 $\pm 4.4$&  19.50 $\pm 4.4$ & \textbf{25.66} $\pm 1.6$ &23.33 $\pm 3.4$  \\\hline
      NCI1 & 44.84 &  63.65  & 56.56$\pm 4.5$ &  56.56$\pm 4.5$ & \textbf{ 62.40} $\pm 1.5 $& 61.60$\pm 2.6$\\\hline
      NCI109 & 59.80 & 62.44  & 62.00$\pm 4.0$ & 62.00$\pm 4.0$&\textbf{  62.15} $\pm 1.8$ & 58.32$\pm 4.6$   \\
      \hline
    \end{tabular}
  \end{center}
\end{table*} 

To compare the efficacy of our approach, we compare our Kneser-Ney and
Pitman-Yor smoothed kernels with state-of-the-art graph kernels namel
the graphlet kernel \cite{SheVisPetMehetal09}, the hash kernel
\cite{ShiPetDroLanetal09}, the random walk kernel \cite{GaeFlaWro03},
\cite{KasTsuIno04}, \cite{VisBorSch07b}, and the shortest path kernel
\cite{BorKri05}. For random walk kernel, we uniformly set the decay
factor $\lambda = 10^{-4}$, for shortest path we used the delta kernel
to compare the shortest-path distances, and for the hash kernel we used
a prime number of 11291.  We adopted Markov chain Monte Carlo sampling
based inference scheme for the hierarchical Pitman-Yor language model
from \cite{Teh06b} and modified the open source implementation of HPYP
from \url{https://github.com/redpony/cpyp}.  Due to lack of space we
will only present a subset of our experimental results. Full results
including the source code and experimental scripts will be made
available for download from \url{http://cs.purdue.edu/~ypinar/kdd}.

\subsection*{Datasets}

In order to test the efficacy of our model, we applied smoothing to
real-world benchmark datasets, namely MUTAG, PTC, NCI1, NCI109, ENZYMES
and DD. MUTAG \cite{DebLopDebShuetal91} is a binary data set of 188
mutagenic aromatic and heteroaromatic nitro compounds, labeled whether
they have mutagenicity in \textit{Salmonella typhimurium}. The
Predictive Toxicology Challenge (PTC) \cite{ToiSriKinKraetal03} dataset
is a chemical compound dataset that reports the carcinogenicity for male
and female rats. NCI1 and NCI109 \cite{WalWatKar08},
(http://pubchem.ncbi.nlm.nih.gov) datasets, made publicly available by
the National Cancer Institute (NCI), are two subsets of balanced data
sets of chemical compounds screened for ability to suppress or inhibit
the growth of a panel of human tumor cell lines. Enzymes is a data set
of protein tertiary structures obtained from
\cite{BorOngSchVisetal05b}. DD \cite{DobDoi03} is a data set of protein
structures where each protein is represented by a graph and nodes are
amino acids that are connected by an edge if they are less than 6
Angstroms apart. Table \ref{table-datasets} shows summary statistics for
these datasets. Note that we did not use edge or node labels in our
experiments.

\begin{table}[t]
  \caption{Properties of the datasets}
  \label{table-datasets}
  \begin{center}
    \begin{tabular}{|c|c|c|c|c|}
      \hline
      Dataset & Size & Classes & Avg Nodes & Avg Edges  \\
      \hline
      MUTAG & 188 & 2 (125 vs 63) & 17.9 & 39.5\\ \hline
      PTC & 344 & 2 (152 vs 192)  & 25.5  &  51.9\\ \hline
      Enzyme & 600 & 6 (100 each)  & 32.6  & 124.2\\ \hline
      DD & 1178 & 2 (691 vs 487) & 284.3  &  1431.3\\ \hline
      NCI1 & 4110 & 2 (2057 vs 2053) & 29.8 & 64.6  \\\hline
      NCI109 & 4127 & 2 (2079 vs 2048) & 29.6 &  62.2\\ 
      \hline
    \end{tabular}
  \end{center}
\end{table}

\subsection*{Experimental Setting}

All data sets we work with consist of sparse graphs. However, counting
all graphlets of size $k$ for a graph with $n$ nodes requires $O(n^k)$
effort which is intractable even for moderate values of $k$. Therefore,
we use random sampling, as advocated by \cite{SheVisPetMehetal09}, in
order to obtain an empirical distribution of graphlet counts that is
close to the actual distribution of graphlets in the graph. For each
value of for each $k \in \{2,\ldots,8\}$ we randomly sampled 10,000
sub-graphs, and used Nauty \cite{McKay07} to get canonically-labeled
isomorphic representations which are then used to construct the
frequency representation. 

We performed 5-fold cross-validation with C-Support Vector Machine
Classification using LibSVM \cite{ChaLin01b}, using 4 folds for training
and 1 for testing. We used a linear kernel. Since we are interested in
understanding the difference in performance between smoothed and
un-smoothed kernels we did not tune the value of $C$; it was simply set
to 1. In order to tune the discount parameter for Kneser-Ney based
smoothed kernel, we tried different parameters vary from 0.01 to 10,000
and report results for the best one.

\begin{figure}[ht]
  \vskip 0.2in
  \begin{center}
    \centerline{\includegraphics[width=1.0\columnwidth]{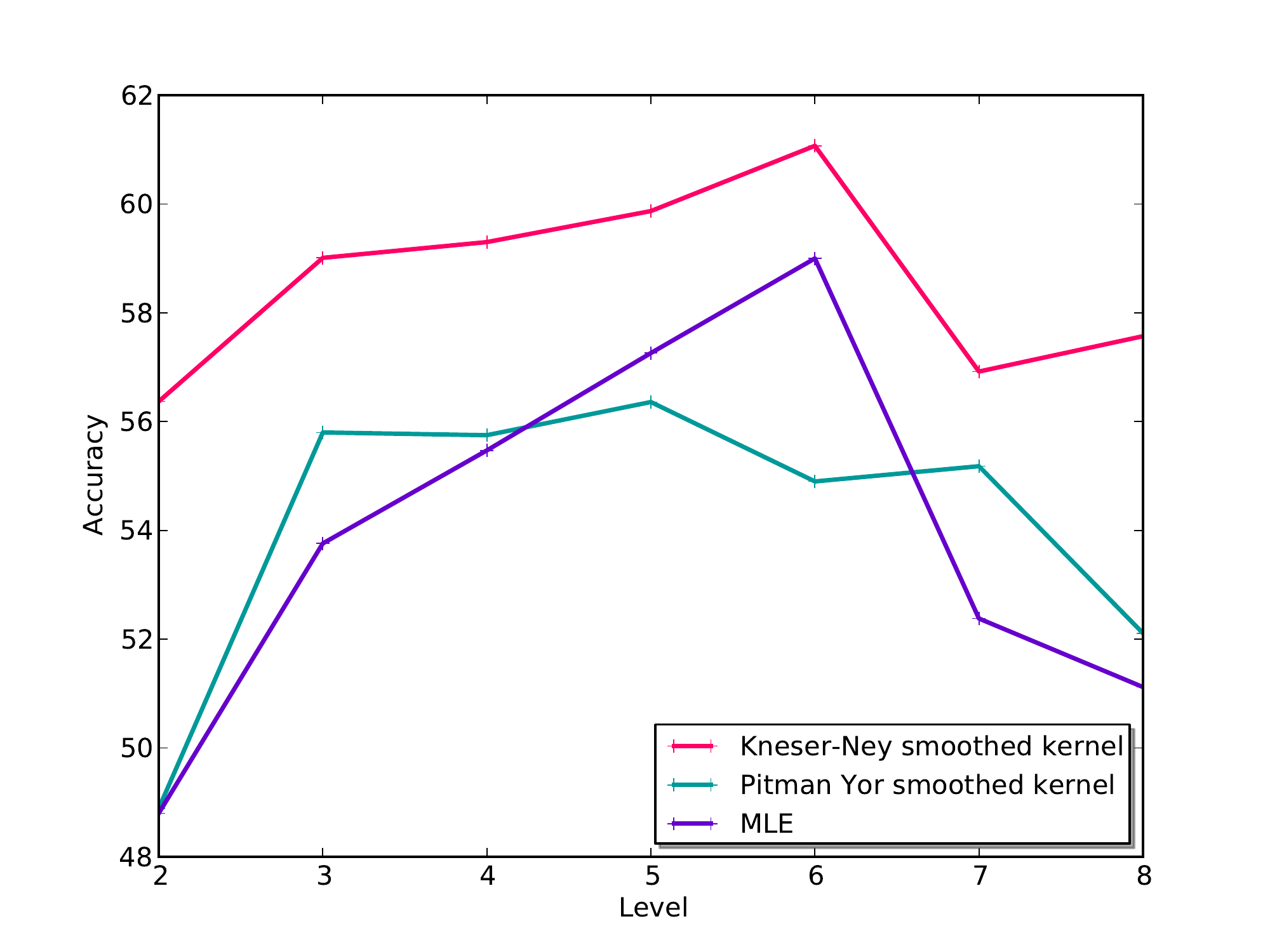}}
    \caption{
      Accuracy vs. graphlet size for PTC dataset with MLE, PYP and Kneser-Ney smoothed kernels}
    \label{fig:n-vs-acc}
  \end{center}
  \vskip -0.2in
\end{figure} 

\begin{table}[t]
  \caption{Graphlet Kernel with exhaustive sampling vs. Smoothed Kernel with 10,000 samples}
  \label{tab:results-bruteforce}
  \begin{center}
    \begin{tabular}{|c|c|c|c|l|}\hline
      Dataset & GK  & KN & PYP \\ \hline  
      MUTAG & 80.85 & \textbf{82.98} $\pm 5.4$ & 81.94$\pm 2.9$\\ \hline 
      PTC & 54.94 & \textbf{59.87}$\pm 4.6$ & 56.36$\pm 4.3$\\ \hline 
      DD & > 48h &  \textbf{74.95}$\pm 2.3 $& 73.51$\pm 2.8$\\ \hline 
      ENZYMES & > 48h & \textbf{25.66} $\pm 1.6$ &23.33 $\pm 3.4$  \\ \hline 
      NCI1 & 62.36&  \textbf{62.40 }$\pm 1.5 $& 61.60$\pm 2.6$\\ \hline 
      NCI109 &\textbf{62.53} & 62.15 $\pm 1.8$ & 58.32$\pm 4.6$  \\  
      \hline
    \end{tabular}
  \end{center}
\end{table}

\subsection*{Effect of discounting parameter on performance}

First, we investigate the effect of the discounting parameter on the
classification performance. Since the trends are similar across
different datasets, we pick PTC as a representative dataset to report
results. Figure \ref{fig:discount-vs-acc} shows the classification
accuracy on the PTC dataset with different discounting parameters for
Kneser-Ney smoothing. As expected, applying very large discounts
decreases the performance because the distribution
\eqref{eq:kn-graphlet} degrades to the base distribution.  On the other
hand, applying a very small discount also decreases the accuracy since
the distribution degrades to the MLE estimate. From our experiments, we
observe that the best performance in all datasets is achieved by using
an intermediate discount value between these two extremes. However, the
specific discount value is data dependent.


\begin{figure}[ht]
  \vskip 0.2in
  \begin{center}
    \centerline{\includegraphics[width=1.0\columnwidth]{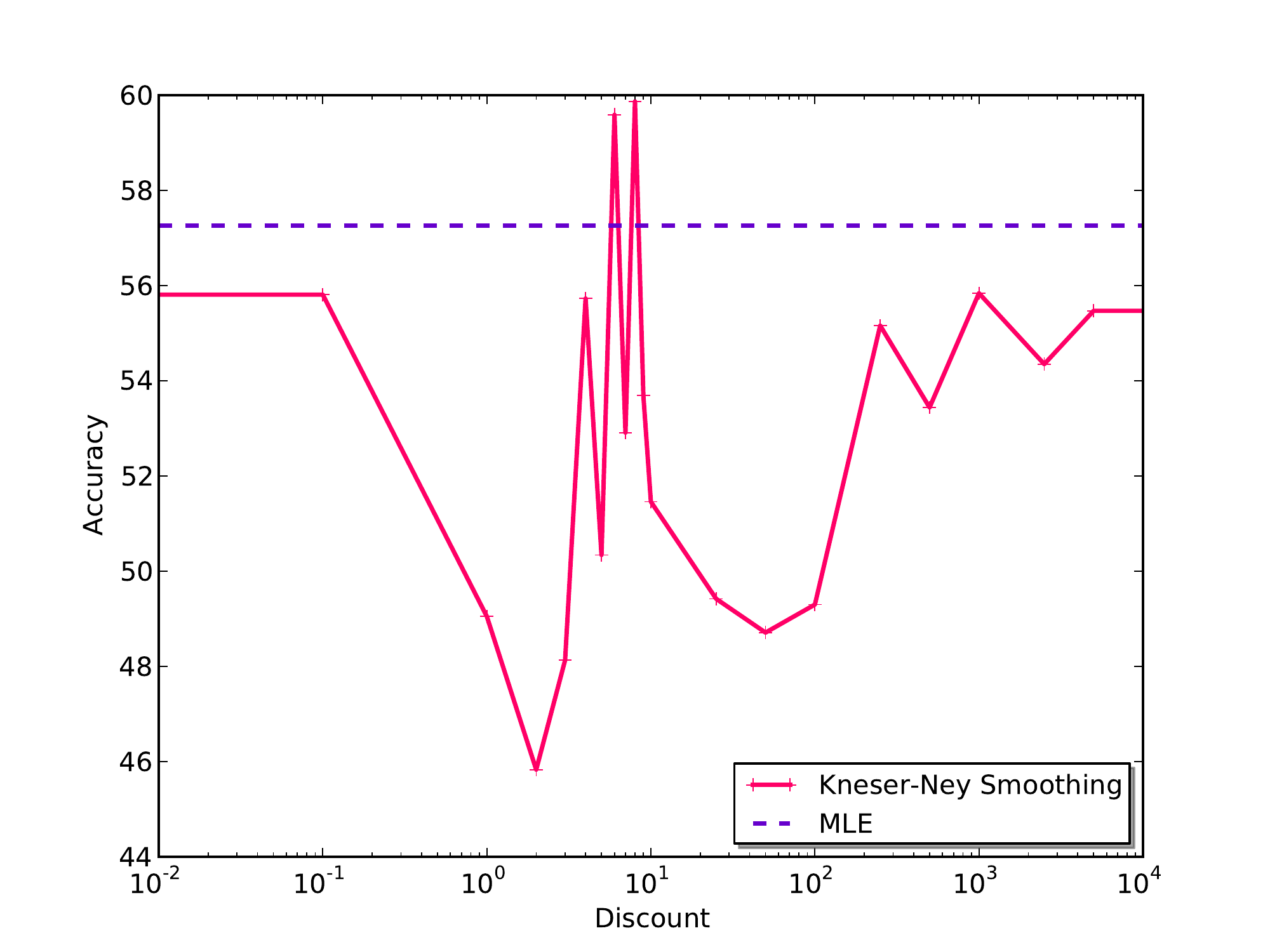}}
    \caption{
      Accuracy vs. Discounting for PTC dataset, k=5 graphlets }
    \label{fig:discount-vs-acc}
  \end{center}
  \vskip -0.2in
\end{figure} 

\subsection*{Effect of graphlet size on performance}

Next we investigate how the value of $k$ affects performance. Again, we
show results for a representative dataset namely PTC. Figure
\ref{fig:n-vs-acc} shows the classification accuracy on the PTC dataset
as a function of graphlet size $k$ for MLE (the graphlet kernel),
Pitman-Yor smoothed kernel and Kneser-Ney smoothed kernel.  From the
figure, we can see that small graphlet sizes such as $k=2,3$ do not
perform well and are not informative since the number of unique
graphlets is very small (see Table \ref{fig:unique-graphlets}). On the
other hand, MLE does not perform well for large graphlet sizes such as
$k=7,8$ because of the diagonal dominance problem.  On the other hand,
smoothed kernels in general obtain a balance between these two extreme
situations and tend to yield better performance. Pitman-Yor smoothed kernel 
tends to achieve a better performance than MLE, but doesn't perform as good 
as Kneser-Ney. This is expected since we didn't tune the hyperparameters 
for Pitman-Yor process. \citet{Teh06b} shows that Pitman-Yor yields a better 
performance if one tune the hyperparameters. Therefore, our Pitman-Yor 
kernel is open to improvement.

\subsection*{Comparison with related work}

We compare the proposed smoothed kernel with graphlet kernel and hash
kernel on the benchmark data sets in Table \ref{tab:results}. The
results for the Shortest Path and Random Walk graph kernels are included
mainly to show what is the state-of-the-art using other
representations. We fixed the $k=5$ which is observed to be the best
value for the MLE based graphlet kernel on most datasets. We randomly
sample 10,000 graphlets from each graph and feed the same frequency
vectors to graphlet kernel (GK), hash kernel (HK), Kneser-Ney smoothed
kernel (KN), and Pitman-Yor smoothed kernel (PYP). Therefore, the
differences in performance that we observe are solely due to the
transformation of the frequency vectors that these kernels perform. We
performed an unpaired $t$-test and use bold numbers to indicate that the
results were statistically significant at $p < 0.0001$.

We can see that KN kernel outperforms MLE and hash kernels on all of the
benchmark data sets. The accuracy of the PYP kernel is usually lower
than that of the KN kernel. We conjecture that this is because the
Pitman-Yor process is sensitive to the hyper-parameters and we do not
carefully tune the hyper-parameters in our experiments. On PTC, DD and
Enzymes datasets, the KN kernel reached the highest accuracy. On MUTAG,
NCI1 and NCI109 datasets, KN kernels also yield good results and got
comparable classification accuracies to shortest path and random walk
kernels. For the DD dataset, shortest path and random walk kernels were
not able to finish in 24 hours, due to the fact that this dataset has a
large maximum degree.

To summarize, smoothed kernels turns out to be competitive in terms of
classification accuracy on all datasets and are also applicable to very
large graphs.

\subsection*{Effect of exhaustive sampling on performance}

Next, we investigate whether any of the difference in performance can be
attributed to sampling a small number of graphlets. In other words, we
ask do the results summarily change if we performed exhaustive sampling
instead of using 10,000 samples. We give MLE an unfair advantage by
performing exhaustive sampling on MUTAG, PTC, NCI and NCI109 datasets
for $k=5$ by using a distributed memory implementation. Table
\ref{table-bruteforcesampling} shows mean, median and standard
deviations of number of samples in bruteforce sampled datasets for
$k=5$. Here, we can see that the original frequencies of the graphlets
are quite high in most of the datasets. Even though our algorithm only
uses 10,000 samples, it outperforms the graphlet kernel with exhaustive
sampling on MUTAG, PTC and NCI1 and achieves competitive performance on
NCI109 dataset. Results are summarized in Table
\ref{tab:results-bruteforce}.


Even though distribution with a small number of samples is close to the
original distribution in the $L_{1}$ sense, bruteforce sampling reveals
that the true underlying distribution of the datasets contains a larger
number of unique graphlets comparing to random sampling. Since the graphlet
kernel uses a MLE estimate its performance degrades. On the other hand,
our smoothing technique uses structural information to redistribute the
mass and hence is able to outperform MLE even with a small number of
samples. 

\begin{table}[t]
  \caption{Mean, Std and Median number of size $k=5$ graphlets per graph.}
  \label{table-bruteforcesampling}
  \begin{center}
    \begin{tabular}{|c|c|c|c|l|}\hline
      Dataset & Mean  & Std & Median \\ \hline  
      NCI109 & 1128404.0613 & 5330478.99592 & 65780.0 \\ \hline
      NCI1 & 1164713.35912 & 5896578.49179 & 80730.0 \\ \hline
      PTC & 732605.578488 & 3243397.46737 & 26334.0 \\ \hline
      MUTAG & 16221.1702128 & 19687.3766842 & 7378.0 \\  
      \hline
    \end{tabular}
  \end{center}
\end{table}



\section{Discussion}
\label{sec:Discussion}

We presented a novel framework for smoothing normalized
graphlet-frequency vectors inspired by smoothing techniques from natural
language processing. Although our models are inspired by work done in
language models, they are fundamentally different in the way they define
a fallback base distribution. We believe that our framework has
applicability beyond graph kernels, and can be used in any structural
setting where one can naturally define a relationship such as the DAG
that we defined in Figure \ref{fig:dag}. We are currently investigating
the applicability of our framework to string kernels
\citep{VisSmo03,LesEskNob02,LesKua03} and tree kernels
\citep{ColDuf01}. It is also interesting to investigate if our method
can be extended to other graph kernels such as random walk kernels.  Our
framework is also applicable to node-labeled graphs since they also
suffer from similar sparsity issues. We leave the application of our
framework to labelled graphs to an extended version of this paper.  We
are also investigating better strategies for tuning the hyper-parameters
of the Pitman-Yor kernels.

\bibliography{graphlet_paper}
\end{document}